%% file: main.tex
\RequirePackage{amsmath}
\documentclass[runningheads]{llncs}
\usepackage[T1]{fontenc}
\usepackage{amssymb}
\usepackage{algorithm}
\usepackage{algpseudocode}
\usepackage[hyphens]{url}            
\usepackage{hyperref}       
\usepackage{orcidlink} 

\usepackage{graphicx}
\usepackage{makecell}
\usepackage{booktabs}
\usepackage{xcolor}
\usepackage{multirow}
\usepackage[false]{anonymous} 

\usepackage{microtype}

\usepackage{cancel}

\graphicspath{ {./images/} }

\usepackage{color}

\newif\ifshowcomments
\showcommentstrue
\ifshowcomments
    \newcommand{\mynote}[2]{\fbox{\bfseries\sffamily\scriptsize{#1}}{\small$\blacktriangleright$\textsf{#2}$\blacktriangleleft$}}
\else
    \newcommand{\mynote}[2]{}
\fi

\begin{document}
\title{The WHY in Business Processes:\\ Unification of Causal Process Models}

%
%
\authoranon{
\author{
Yuval David\orcidlink{0009-0003-0282-7394} \and
Fabiana Fournier\orcidlink{0000-0001-6569-1023} \and
Lior Limonad\orcidlink{0000-0002-4784-2147} \and
Inna Skarbovsky\orcidlink{0000-0002-9398-4373}
}
\authorrunning{Y. David et al.}
%
\institute{IBM Research, Israel\\
\email{\{yuval.david,fabiana,liorli,inna\}@il.ibm.com}}
}
\maketitle
\begin{abstract}
Causal reasoning is essential for business process interventions and improvement, requiring a clear understanding of causal relationships among activity execution times in an event log. Recent work introduced a method for discovering causal process models but lacked the ability to capture alternating causal conditions across multiple variants.
This raises the challenges of handling missing values and expressing the alternating conditions among log splits when blending traces with varying activities.

We propose a novel method to unify multiple causal process variants into a consistent model that preserves the correctness of the original causal models, while explicitly representing their causal-flow alternations. The method is formally defined, proved, evaluated on three open and two proprietary datasets, and released as an open-source implementation.
\end{abstract}


\section{Introduction}

``A rooster's crow does not cause the sun to rise, even though it always precedes the sun.''~\cite{Pearl2018Why} But will a rooster’s crowing pattern remain the same if placed at the North Pole? While research by~\cite{Shimmura2013CircadianCrowing} has shown that a rooster’s crow is influenced by its biological circadian cycle and may continue even when placed in a dark room, causal diagrams allow us to predict the effects of interventions without actually conducting an experiment~\cite{Pearl2018Why}. 
Our aim is to perform the same type of analysis in the context of business processes, eliminating the need to intervene directly in the process, which is highly costly.
\textanon{Our developed}{The} causal execution business process (CBP) model~\cite{Fournier2023v3} links activity executions, enabling the assessment of how omitting or modifying an activity’s execution time (e.g., delaying or expediting it) impacts other activities. This eliminates the need for empirical experimentation while informing resource allocation decisions.

Causal inference and causal discovery (CD) are the main pillars of causal analysis. While causal discovery focuses on analyzing and creating models that illustrate the relationships inherent in the data~\cite{Peters2017ElementsAlgorithms,Pearl2011Causality:Edition,Spirtes2001CausationSearch}, causal inference is the process of drawing a conclusion about a causal connection based on the conditions of the occurrence of an effect~\cite{Yao2021AInference,Cunningham2021CausalMixtape,Hernan2020CausalIf}.
More concretely, CD aims at constructing causal graphs from data by exploring hypotheses about the causal structure~\cite{Shimizu2022StatisticalApproach}.
Additional assumptions, such as functional forms and distributions, are often required to identify the causal graph from the data. 
In typical CD settings, the causal graph is assumed to be a Directed Acyclic Graph (DAG), with all common causes of observed variables also being observed (i.e., present in the event log).

The work in~\cite{Fournier2023v3} laid the foundation for discovering CBP models using process execution times. Given multiple CBP models, each capturing a causal perspective of activity execution, unifying them into a single cohesive representation remains a challenge, referred to as the `CBP model unification' problem. Process variability often leads to differing or even conflicting causal conditions. For example, in a loan application process, one variant may follow a sequential flow: after preliminary screening (A), a credit check (B) determines the loan amount (C). In high-demand periods, the process may be expedited, where screening (A) directly determines (C). Symbolically representing the two variants, the first states `A' causes `B', and `B' causes `C', while the second states `A' causes `C' directly.

In~\cite{Fournier2023v3}, a simple union of graph elements (i.e., nodes and edges) was proposed to merge multiple CBP models into a more holistic view. While valid, this approach does not account for the alternating causal execution sequences across variants. A na\"ive union of CBP models would state that `A' causes both `B' and `C', without distinguishing between cases where \texttt{sometimes} `A' causes `B' and \texttt{sometimes} `A' causes `C'. Similarly, it fails to accurately represent the causes of `C'. Even in this simple example, it is clear that the current CBP model lacks the semantics to capture the alternating causal conditions that unfold across different variants.

Another approach could merge the two variants by bundling their traces, where each trace records a sequence of activity executions in the event log. However, this results in an inconsistent data structure, with missing timestamp values across different execution sequences. In the earlier example, traces from the expedited variant would lack a timestamp for activity `B'. This poses a challenge for causal discovery, which struggles with \textit{missing data}~\cite{Conforti2017FilteringLogs}. Traditional causal discovery algorithms typically assume complete data—an assumption often violated in practice~\cite{Mohan2013MissingProblem}. Missing data can bias estimations, distort causal structures, and lead to flawed decision-making~\cite{Little2014StatisticalData}.
We note that this problem is unique to CBP model discovery, due to the inability to run core causal discovery algorithms on event logs containing null values, unlike conventional process mining techniques, which can still generate models such as the Directly Follows Graphs (DFG). This limitation necessitates unification at the model level, rather than through trace-level integration. Consequently, the contribution presented here also enhances the robustness of the CBP model discovery approach by addressing its vulnerability to missing data.



Our work addresses the CBP model unification problem, substantially extending~\cite{Fournier2023v3} by introducing a framework for unifying multiple CBP models, each corresponding to a process variant, into a holistic graph representation. This framework preserves the causal knowledge in each CBP model, explicitly captures the logic of alternating execution conditions, and accounts for missing data.

We introduce a novel method that integrates multiple CBP models into a cohesive causal execution model while preserving the individual causal dependencies in the underlying variants, and articulating the conditions under which different execution paths occur. To achieve this, we augment the causal process model with a gating mechanism that encodes the alternating logic among variants. The method is formalized and assessed. We disclose its computational properties, prove its correctness and empirically evaluate its performance and scalability using five benchmark datasets, including three open and two proprietary ones.

\section{Preliminaries}


This section outlines key concepts essential for understanding the proposed framework. Let $L$ be an input process event log, where $L$ is defined as follows:

\begin{definition}
A \textit{process event log} \( L \) is defined as a finite set of traces, where each trace represents an ordered sequence of activities. Formally, \( L = \{ \tau_1, \tau_2, \dots, \tau_n \} \), where \( \tau_i \) is a trace, and \( n \) is the total number of traces in the log. Each trace \( \tau \) is associated with a unique identifier, \( case_{id} \), and is an ordered sequence of activities, \( \tau = \langle a_1, a_2, \dots, a_m \rangle \), where \( a_j \) is an activity, and \( m \) is the total number of activities in the trace. An activity \( a \) is a tuple \( a = (case_{id}, t, n_{id}, V) \), where \( case_{id} \) is a unique identifier linking the activity to its corresponding trace, \( t \) is the timestamp indicating when the activity occurred (could be start, end, or both), \( n_{id} \) is the unique name of the activity, and \( V \) is a (possibly empty) set of additional `payload' attributes \( \{v_1, v_2, \dots, v_k\} \). Each attribute \( v_k \) is a key-value pair, defined as \( v_k = (k, v) \), where \( k \) is the key representing the attribute name, and \( v \) is the value associated with the key. 
\end{definition}

The timestamps \( t \) in a trace \( \tau \) must be non-decreasing (\( t_1 \leq t_2 \leq \dots \leq t_m \)), and activity names \( n_{id} \) must belong to a predefined set of activity labels \( N \) (\( n_{id} \in N \)). The attributes \( V \) may vary between activities and traces. 

Assume the following event log $L$:
\begingroup
\footnotesize
\begin{align*}
L = 
& (\tau_1=\{A^1, B^3,C^6\}, \tau_2=\{A^2, F^5\}, \tau_3=\{F^4, G^8,H^{12}\}, 
\tau_4=\{A^{10}, F^{15}\},\\ & \tau_5=\{A^{13}, C^{14}, B^{17}\}), 
\text{ where superscript numbers correspond to timestamps.}
\end{align*}
\endgroup


For the given log $L$, it may be further partitioned into a set of variants, where a variant $v$ is defined as follows:

\begin{definition}
A \textit{variant} \( v \) in an event log \( L \) is a subset of traces in \( L \), where all traces in \( v \) share the same ordered sequence of activities. Formally, for an event log \( L = \{\tau_1, \tau_2, \dots, \tau_n\} \), a variant \( v \subseteq L \) is defined as \( v = \{\tau \in L \mid \tau = \langle a_1, a_2, \dots, a_m \rangle \text{ and } \tau' = \langle a_1, a_2, \dots, a_m \rangle, \forall \tau, \tau' \in v\} \), where \( \tau \) and \( \tau' \) are traces, and \( \langle a_1, a_2, \dots, a_m \rangle \) denotes the ordered sequence of activities in the trace. The set of all unique variants in \( L \) is denoted by \( V_L \), where \( V_L = \{v_1, v_2, \dots, v_k\} \), such that \( \bigcup_{i=1}^{k} v_i = L \) and \( v_i \cap v_j = \emptyset \) for \( i \neq j \). That is, each variant \( v \) groups traces from \( L \) that represent identical sequences of activities.
\end{definition}

For the above example log $L$, the variants in the log are the following:
\begingroup
\footnotesize
\begin{align*}
V_L = 
& (v_1=\{A, B, C\}_1, v_2=\{A, F\}_{2,4}, v_3=\{F, G, H\}_3, v_4=\{A, C, B\}_5),\\
& \text{where subscript numbers correspond to case IDs.}
\end{align*}
\endgroup

A less dense partitioning of the log can be derived by combining all variants sharing the same set of activities, independent of activity ordering. Thus, we define a partition $p$ as follows:

\begin{definition}
A \textit{partition} \( p \) in an event log \( L \) is a subset of traces in \( L \), where all traces in \( p \) share the same set of activities, independent of their ordering. Formally, for an event log \( L = \{\tau_1, \tau_2, \dots, \tau_n\} \), a partition \( p \subseteq L \) is defined as \( p = \{\tau \in L \mid \text{set}(\tau) = \text{set}(\tau'), \forall \tau, \tau' \in p\} \), where \( \text{set}(\tau) \) represents the set of activities in the trace \( \tau \) and ordering is disregarded. 
The set of all partitions in \( L \) is denoted by \( P_L \), where \( P_L = \{p_1, p_2, \dots, p_m\} \), such that \( \bigcup_{i=1}^{m} p_i = L \) and \( p_i \cap p_j = \emptyset \) for \( i \neq j \). 
\end{definition}

For the above example log $L$ and variants $V_L$, log partitions are the following:
\begingroup
\footnotesize
\begin{align*}
P_L = 
& (p_1=\{1, 5\}, p_2=\{2, 4\}, p_3=\{3\}),
 \text{ where numbers correspond to case IDs.}
\end{align*}
\endgroup

\textanon{In our work in}{In}~\cite{Fournier2023v3}, an algorithm for causal business process model discovery was presented that identifies all inter-activity relations implying causal execution dependencies between the activities. The core algorithm presented adapts the LiNGAM causal discovery method to timestamped event logs, treating activity execution times as observed variables to infer a directed acyclic graph (DAG) of causal dependencies between activities. Given a process event log as an input, the developed algorithm could be applied to any process variant or a set of variants. The result can be represented as a causal execution (CX) graph $G$ defined as follows:

\begin{definition}
A \textit{causal execution graph} \( G \), resulting from applying a causal discovery algorithm to a process event log \( L \), is a tuple \( G = (V, E) \), where \( V \) is a finite set of nodes, and each node represents an activity name occurring in \( L \). The set of edges \( E \subseteq V \times V \) denotes \textit{causal execution relationships}, where each edge \( (n_i, n_j) \in E \) signifies that the execution of the activity corresponding to \( n_i \) causes the execution of the activity corresponding to \( n_j \). 
\end{definition}

We assume that the name of an activity is unique in the event log (or the activity can be uniquely identified by some other attribute).  The graph \( G \) can be derived from a process variant (a single sequence of traces), a partition (a set of traces with the same set of activities, independent of order), or the entire event log \( L \), and it provides a representation of the causal structure of the process with nodes as activities and directed edges as causal execution dependencies.

Corresponding to the above example partitions, Figure~\ref{fig:causal-graphs-example} depicts possible causal execution graphs corresponding to each partition.

\begin{figure}[ht]
    \centering
    \begin{center}  
    \vspace{-0.5cm}
        \begin{tabular}{c c c} 
            \includegraphics[width=0.2\linewidth]{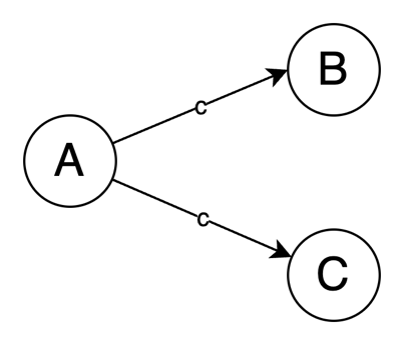} &
            \raisebox{0.8\height}{\includegraphics[width=0.2\linewidth]{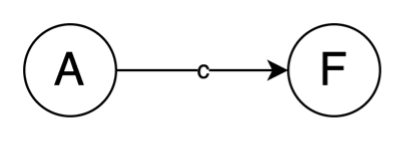}} &
            \includegraphics[width=0.2\linewidth]{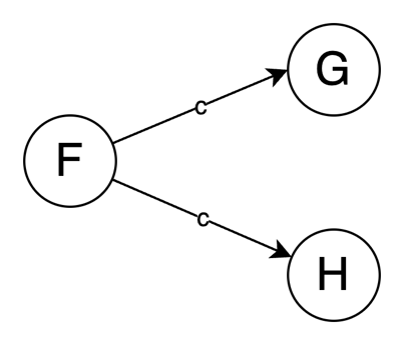} \\
            \parbox{0.3\linewidth}{for $p_1$: $g_1$ = \\ ( N=\{A,B,C\},\\ E=\{(A,B),(A,C)\} )} & 
            \parbox{0.3\linewidth}{for $p_2$: $g_2$ = \\ ( N=\{A,F\},\\ E=\{(A,F)\} )} & 
            \parbox{0.3\linewidth}{for $p_3$: $g_3$ = \\ ( N=\{F,G,H\},\\ E=\{(F,G),(F,H)\} )} \\
        \end{tabular}
    \end{center}
    \vspace{-0.3cm}
    \caption{Three causal execution graphs for the above example log $L$ and partitions $P_L$}
    \vspace{-0.3cm}
    \label{fig:causal-graphs-example}
\end{figure}

\section{Framework}

This section presents the developed method, its underlying approach and assumptions, and concludes with a proof of its correctness.

\subsection{Our Approach}

While the work in~\cite{Fournier2023v3} presented graph unification as a means to combine multiple causal execution graphs that correspond to process variants, it introduced a na\"ive method employing a simple logical union over the input CXs. This yields a causal graph that is semantically ambiguous concerning the alternating causal execution sequences among the variants. To generate a unified causal execution (U-CX) graph it is essential to address two challenges:
\begin{enumerate}
\vspace{-0.1cm}
    \item The need to accommodate causal process discovery employment to missing values in the data resulting from blending traces with different activities.
    \item The need to express the alternating causal execution conditions among the different log splits.
\vspace{-0.1cm}
\end{enumerate}

To tackle the first challenge, we identify a plausible log split that relies on as many observations as possible. This is accomplished by adhering to the assumption that it is possible to combine traces of different variants as long as no known confounder may account for any differences among them. Consequently, we assume:
\newtheorem{assumption}{Assumption}
\begin{assumption}
\vspace{-0.2cm}
    Within each log split, there are no alternating causal execution conditions between any subset of variants that correspond to the same set of activities in the given event log.
\vspace{-0.2cm}
\end{assumption}

Splitting the log into \textit{partitions} allows the application of the CD algorithm developed in~\cite{Fournier2023v3}, yielding a corresponding CX graph for each.
However, in cases where there are some known confounders to induce alternating causal conditions within a partition, we let the user override the default split into more fine-grain partitioning, yielding a split that adheres to the aforementioned assumption. For example, in a loan approval process, the \textit{credit assessment} activity typically determines the \textit{loan amount approved}. Given only these two activities and no additional context, it is reasonable to assume a causal relationship between them. However, an exception may arise in scenarios where the typical adjacency is disrupted — for instance, in organizations where there is a habitual break between 2 p.m. and 4 p.m., the approval activity might consistently occur after the break, regardless of when the assessment finishes. In such cases, the temporal ordering no longer implies causality. To account for these situations, we allow the user to override the default assumption.

Addressing the first challenge enables applying the CD algorithm within each partition but entails adhering to its assumptions, as required by the LiNGAM adaptation~\cite{Fournier2023v3}. These include treating activity execution times as continuous, non-Gaussian, linearly dependent variables with no unobserved confounders.

Given that some degree of noisy measurement is also captured in real-world datasets, some mechanism to mitigate the effect of noise must be embedded to correctly identify the conditions within each partition. For this, we also include a thresholding component in the method to cope with the presence of noise.

To address the second challenge, after identifying splits, we require a formal symbolic representation and an algorithmic method to capture the alternating causal execution conditions across partitions. For this, we extend the basic CX graph with a notation to express logical alternations in causal relations and develop the method detailed in Section~\ref{sec:method}.




Regarding the extension of the symbolic representation, we incorporated four new types of relationships that connect three or more activities (see example in Figure~\ref{fig:four-gateways}). We introduce these new relationship types as follows. A formal definition is available in Appendix~\ref{sec:gateways}.

\begin{itemize}
    \item \textbf{Causal ``And'' Gateway (AND\(_C\))}: This gateway represents a situation where an activity execution always causes multiple other activities. Whenever the causing activity occurs, all connected activities must also occur.

    \item \textbf{Causal ``Or'' Gateway (OR\(_C\))}: This gateway represents a situation where an activity execution can lead to one or more different activities, but not necessarily all of them. At least one of the connected activities must occur, but it is not required that all do.

    \item \textbf{Exhaustive Causal ``Or'' Gateway (OR\(_C^E\))}: This is a special case of the ``Or'' gateway, where the causing activity can trigger any possible combination of the connected activities. That means it could cause just one, several, or even all of them to occur. We introduce this type of gateway for notational clarity, serving as syntactic sugar.

    \item \textbf{Causal ``Xor'' Gateway (XOR\(_C\))}: This gateway represents a situation where an activity execution leads to exactly one of several possible activities, but never more than one. When the initial activity occurs, only one of the connected activities will take place.
\end{itemize}

Mirroring the above (causal-split) gateways by inverting their source and target edges, we also extend the notation with causal-join gateway types: $AND_{C>}$, $OR_{C>}$, $OR_{C>}^E$, and $XOR_{C>}$. Based on these definitions, we thus define:


\begin{definition}
An \textit{Extended Causal Execution (CX) Graph} is a causal execution graph \( G = (V, E) \), where \( V = V_A \cup V_G \) is the set of nodes and \( E \subseteq V \times V \) is the set of edges. The set \( V_A \) represents activity nodes corresponding to activities in the process, and the set \( V_G \) represents gateway nodes, which include \( \text{AND}_C \), \( \text{OR}_C \), \( \text{XOR}_C \), and \( \text{OR}_C^E \), of both split and join gateway types.
\end{definition}
\begingroup
\vspace{-0.6cm}
\setlength{\tabcolsep}{4pt} 
\renewcommand{\arraystretch}{1.2} 
\begin{figure}
    \centering
    \begin{tabular}{c|c}
        \includegraphics[width=0.48\linewidth]{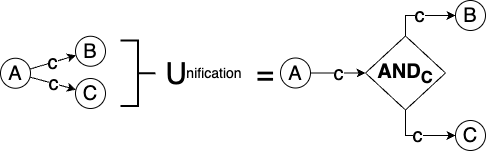} & \includegraphics[width=0.48\linewidth]{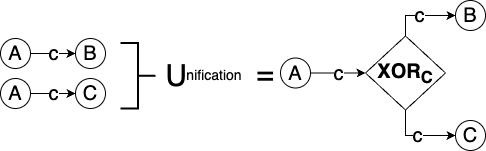}\\
        $AND_C$ gateway unification example & $XOR_C$ gateway unification example\\
        \hline\\
        \includegraphics[width=0.48\linewidth]{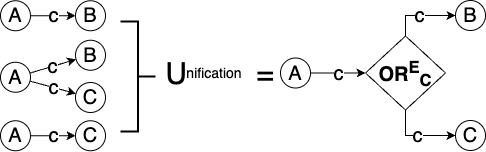} & \includegraphics[width=0.48\linewidth]{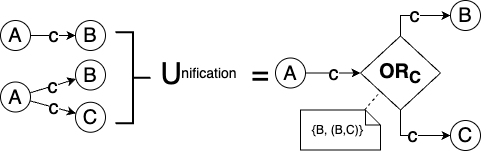}\\
        $OR_C^E$ gateway unification example & $OR_C$ gateway unification example\\
    \end{tabular}
    \caption{Four types of causal split gateways}
    \vspace{-0.4cm}
    \label{fig:four-gateways}
\end{figure}
\endgroup

\subsection{Causal Graph Unification Method}
\label{sec:method}

The method for causal graph unification takes as input an event log \( L \) and a set of variants \( V_s \subseteq V_L \), where \( V_L \) is the set of all variants in \( L \). The output is a U-CX graph \( G_U = (V_U, E_U) \), which results from the combination of the causal execution graphs corresponding to the individual partitions in the log, exclusively preserving the alternative causal relationships across the selected variants. 


Thus, we define:
\begin{definition}
\textit{Causal Graph Unification} is a function \( f_C \) that takes as input an event log \( L \) and a set of variants \( V_s \subseteq V_L \), where \( V_L \) is the set of all variants in \( L \), and produces as output a unified extended (U-CX) form of a causal execution graph \( G_U = (V_U, E_U) \). Formally, \( f_C: (L, V_s) \to G_U \), where \( G_U = (V_U, E_U) \) is the result of combining the set of causal execution graphs \( \{G_v = (V_v, E_v) \mid v \in V_s\} \) corresponding to the individual partitions in \( L \). 
\end{definition}

The unified graph \( G_U \) integrates the causal relationships and employs the multilateral dependencies (\( \text{AND}_C, \text{OR}_C,  \text{OR}_C^E, \text{XOR}_C \)) to construct a representation that maintains coherence concerning the input selected variants.
Such a graph is deemed valid if it adheres to the following requirements: 
\vspace{-0.3cm}
\paragraph{Soundness.} 
An output unified graph is sound if it has no traversal path in it that has no consistent path (i.e., denoting the same causal sequence) in any of the input causal execution graphs corresponding to the log partitions.
\vspace{-0.3cm}
\paragraph{Completeness.} 
An output unified graph is complete if every traversal path in the input causal execution graphs corresponding to the log partitions has a consistent path in the output graph.


The method consists of the following three high-level steps: (1) Log partitioning, (2) Unification, and (3) Simplification, elaborated next. 

\vspace{-0.4cm}
\subsubsection{Log Partitioning}
\label{sec:log-partitioning}

Careful attention should be given to the log partitioning strategy, ensuring alignment with the assumption that a genuine factor (e.g., an unknown confounder) may account for differences between variants. If no such factor exists, breaking down an identified partition into variants can lead to Berkson's paradox~\cite{Pearl2018Why}, potentially introducing spurious causal execution dependencies in the resulting CX graph. 

By default, for any given set of variants in the input, we split the combined set of traces into partitions to ensure the largest possible set of observations is considered for all activities in the given variants. However, we also allow the user to explicitly select a partitioning modality that retains coherence with the original set of variants. In addition, if not explicitly specified in the input, the algorithm will run over the entire event log.

Because of the way we split the data, each partition will not include any NaN values, which allows us to apply the causal process discovery algorithm (CPD) presented in~\cite{Fournier2023v3} to each partition, yielding a corresponding CX graph. We note that it may also be possible for the user to allow the algorithm to combine partition traces with other partitions that contain the same activities, but only when there is no prior knowledge of expected differences in the causal dependencies between the partitions. This can help mitigating the possibility of having causal inconsistencies between a partition and the overall event-log.
Due to space limitations, a formal specification of the partitioning is described in Appendix~\ref{app:log-partitioning}.

\paragraph{A thresholding component for noise elimination.}

In practice, noise or data errors may cause $B$ to occur before $A$, even if this is rare in a partition. To handle this, we use a thresholding mechanism that tolerates minor noise but blocks causal influences opposing the majority trend. This is achieved by blacklisting edges when violations exceed a set threshold.

We define a user-configurable threshold $\theta$, as the maximum proportion of order violations allowed before they become significant. For each activity pair $(A, B)$, we compute the proportion $p_{B \rightarrow A}$ of cases where $B$ precedes $A$.
If $p_{B \rightarrow A} \leq \theta$, we blacklist the edge from $B$ to $A$, filtering minor noise while preserving meaningful temporal constraints. This mechanism ensures the causal discovery algorithm follows dominant temporal patterns without being misled by anomalies. We then apply the CPD algorithm, incorporating thresholding, as $G_i \gets g_i=\text{causalDiscovery}(p_i)$, producing a CX graph for each partition.
\vspace{-0.3cm}
\subsubsection{Unification}

In this step, we aim to unify the CX graphs generated for all partitions. The unification step consists of three main stages:

\begin{enumerate}
    \item Processing the input causal execution graphs into a matrix representation.
    \item Processing the matrix with a unification algorithm.
    \item Forming the output U-CX graph from the eventual matrix.
\vspace{-0.2cm}
\end{enumerate}

The processing of the input CX graphs into a matrix representation considers an input $G_i$, where each graph $g_i = (N_i, E_i) \in G_i$ is defined by $N_i$, the set of activity nodes, and $E_i$, the set of directed edges $(u, v)$ where $u, v \in N_i$. To represent these graphs as a matrix \textbf{for the case of split gateways}, we first compute the union of all activity nodes across the graphs: $N = \bigcup_{i=1}^n N_i$. A matrix $M$ is initialized such that rows correspond to nodes in $N$, columns correspond to the graphs $g_1, g_2, \dots, g_n$, and each cell $M[u][g_i]$ is initialized as an empty set. The matrix is then populated by iterating over the edges of each graph $g_i$: for each edge $(u, v) \in E_i$, $v$ is added to $M[u][g_i]$. That is, all `children' nodes of $u$ are added to the column corresponding to each graph $g_i$. \textbf{For join gateways}, a second matrix having the same structure is created, having each row populated with the family of `parents' nodes. That is, iterating over the edges of each graph $g_i$: for each edge $(u, v) \in E_i$, $u$ is added to $M[v][g_i]$. Without loss of generality, in the remainder of the paper, we refer to the split case, considering the join case as its mirroring equivalent. The algorithm for matrix construction can be found in Appendix~\ref{sec:matrix-construction}.








For example, given an input set $G_i$ including the graphs in Figure~\ref{fig:causal-graphs-example}.
The union of all activity nodes across the graphs is $N = \{a, b, c, f, g, h\}$, and the corresponding matrix representation $M$ is shown in Table~\ref{tab:two-matrices}(a), where each cell contains the set of child activity nodes for the given node and graph.

\begingroup
\footnotesize

\begin{table}
\caption{(a) Matrix representation input. (b) Matrix representation output. }
\label{tab:two-matrices}
    \centering
    \begin{tabular}{cc}

    \begin{minipage}{0.45\textwidth} 
        \centering
        \begin{tabular}{|c|c|c|c|}
            \hline
            \textbf{Node} & $g_1$ & $g_2$ & $g_3$ \\
            \hline
            a & \{b, c\} & \{f\} & $\emptyset$ \\
            f & $\emptyset$ & $\emptyset$ & \{g, h\} \\
            b & $\emptyset$ & $\emptyset$ & $\emptyset$ \\
            c & $\emptyset$ & $\emptyset$ & $\emptyset$ \\
            g & $\emptyset$ & $\emptyset$ & $\emptyset$ \\
            h & $\emptyset$ & $\emptyset$ & $\emptyset$ \\
            \hline
        \end{tabular}
    \end{minipage}

    &
    
    \begin{minipage}{0.45\textwidth} 
        \centering
        \begin{tabular}{|c|c|c|c|}
            \hline
            \textbf{Node} & $g_1$ & $g_2$ & $g_3$ \\ \hline
            \}a\{ & \{(b, c)\} & \{f\} & $\emptyset$ \\
            b & $\emptyset$ & $\emptyset$ & $\emptyset$ \\
            c & $\emptyset$ & $\emptyset$ & $\emptyset$ \\ 
            d & $\emptyset$ & $\emptyset$ & $\emptyset$ \\
            e & $\emptyset$ & $\emptyset$ & $\emptyset$ \\ 
            f & $\emptyset$ & $\emptyset$ & \{(g, h)\} \\ 
            \hline
        \end{tabular}
    \end{minipage}
    \\
    (a)  & (b) \\
    \end{tabular}
\vspace{-0.5cm}
\end{table}

\endgroup


\textbf{For each row in the matrix}, representing a \textit{family of child sets} for a given node across all graphs, apply the \textit{Unification Algorithm}(see Appendix~\ref{sec:unification-alg}) to classify the relationships among the child sets and mark causal gateways as follows:

\begin{enumerate}
    \item \textbf{$\text{AND}_C$ Identification:} For any child set that does not partially intersect with others, promote it to a single element. This is denoted by enclosing the set in round brackets, e.g., $(a,b)$. Once identified, the promotion is propagated across all child sets in the row to ensure consistency.
    \item \textbf{$\text{XOR}_C$ Check:} 
    \begin{itemize}
        \item Invoked only if the row contains \textbf{two or more child sets}.
        \item If all child sets in the row are mutually exclusive (i.e., no intersection between any two sets), the row is classified as $XOR_C$. This is denoted by having the node (row's name) surrounded by ``\}'' and ``\{'', e.g., $\}a\{$.
    \end{itemize}
    \item \textbf{$\text{OR}_C^E$ (Exhaustive $\text{OR}_C$) Check:}
        \begin{itemize}
        \item Invoked only if the row contains \textbf{two or more child sets}.
        \item If the family of child sets is a \textbf{powerset} of the union of its child sets, the row is classified as $\text{OR}_C^E$. This is denoted by having the node (row's name) surrounded by ``['' and ``]'', e.g., $[a]$.
        \end{itemize}
    \item[] \textbf{Otherwise, default to $\text{OR}_C$:}
        \begin{itemize}
            \item If the row is \textbf{not classified} as $\text{OR}_C^E$, it is classified as $\text{OR}_C$. In such a case, the node (row's name) is surrounded by asterisks, e.g., $*a*$.
            \item To ensure the correctness of the result, along with any $OR_C$ gateway that is concluded, we also store its actual set of alternatives (i.e., the Family set) in a designated map, to allow for the visual annotation of each $OR_C$ gateway with an explicit list of its viable alternatives. This is required to ensure the \textit{soundness} of the result.
        \end{itemize}
\end{enumerate}

Once the marking is determined for a row, no further steps are needed for that row. This process is repeated for all rows in the matrix.
Revisiting the above example matrix, and applying the unification algorithm to it concludes with the matrix shown in Table~\ref{tab:two-matrices}(b).

As a last third step, the output U-CX graph is constructed from the annotated matrix (see Figure~\ref{fig:unified-graph-example}). 
A formal specification of the construction is available in Appendix~\ref{sec:matrix-construction}.

\begin{figure}[htb]
    \centering
    \begin{minipage}{0.49\textwidth}
        \footnotesize
        \begin{align*}
            N_U = & \{a, b, c, d, e, f, g, h, \text{AND}_{C1}, \\ 
                 & \text{AND}_{C6}, \text{XOR}_{C1}\}  \\          
            E_U = \{
            & (a, \text{XOR}_{C1}), (\text{XOR}_{C1}, \\
            & \text{AND}_{C1}),(\text{AND}_{C1}, b), (\text{AND}_{C1}, c),\\
            & (\text{XOR}_{C1}, f), (f, \text{AND}_{C6}),\\ 
            & (\text{AND}_{C6}, g), (\text{AND}_{C6}, h)
            \}
        \end{align*}
    \end{minipage}
    \hfill
    \begin{minipage}{0.49\textwidth}
        \includegraphics[width=\linewidth]{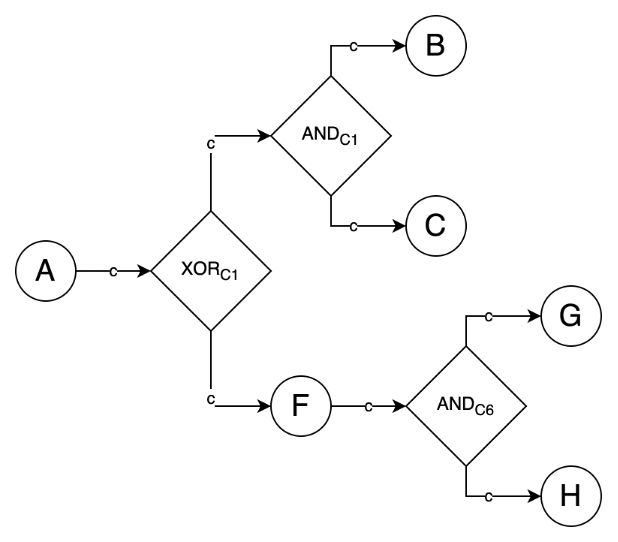} 
    \end{minipage}
    \caption{U-CX graph $G_U$ for the above example log $L$ and graphs: $g_1,g_2,g_3$}
    \vspace{-0.4cm}
    \label{fig:unified-graph-example}
\end{figure}
\subsubsection{Simplification}

This step is optional in the method.
Various heuristics may be employed for further simplification of the non-exhaustive $OR_C$ gateways in the result formed by the unification algorithm.
For example, assume the unification of the following two graphs:
\begingroup
\footnotesize
\[
g_1 = (\{f, a, b\}, \{(f, a), (f, b)\}) \text{ and } g_2 = (\{f, a, c\}, \{(f, a), (f, c)\})
\]
\endgroup

Applying our unification algorithm yields the following:
\begingroup
\footnotesize
\begin{align*}
G_U=(& N_U=\{f,a,b,c,OR_{C1}\},\text{ } 
E_U=\{(f,OR_{C1}),(OR_{C1},a),(OR_{C1},b),(OR_{C1},c)\})\\
& \text{where the } OR_{C1} \text{ alternatives are: } (a,b),(a,c)
\end{align*}
\vspace{-0.5cm}
\endgroup

In the illustrated example, when an OR-type gateway is reached, its set of viable execution alternatives, resulting from the execution of $f$, is restricted exclusively to one of two possibilities, hence being non-exhaustive. Specifically, $f$ either causes the execution of both $a$ and $b$, as implied by $g_1$, or it causes the execution of $a$ and $c$, as implied by $g_2$. Without delving into overly complex formalism, the consequences of executing $f$ can be expressed as: $(a \land b)\oplus(a\land c)$. Such formula is formed as exclusive conjunction of the child nodes of the OR-type gateway.


A variety of logic simplification algorithms can be applied as a final step to simplify such expressions. For example, the same expression can be rewritten as $a \land (b\oplus c)$, preserving logical equivalence. This simplified expression can then be translated back into a graph representation—one that employs a combination of XOR and AND gateways to depict the same causal execution dependencies. However, note that the simplified representation does not necessarily result in a more compact visual graph compared to the original. Therefore, we leave the choice of simplification to the user's discretion.

\vspace{-0.3cm}
\paragraph{Proof of the unification algorithms.} Given a set of input causal execution graphs \( G_i = [G_1, \dots, G_n]\) and a unified, extended, causal execution output graph \( G_U \), a proof of the algorithms' correctness was pursued. This includes the properties of \textit{soundness} and \textit{completeness}. Given a set of input CX graphs and an output U-CX graph, for soundness, we showed that the unification process does not introduce any causal dependencies that are not present in one of the input graphs. This excludes the case of a non-exhaustive $OR_C$ that is sound, if and only if it is annotated with its invocation set. For correctness, we showed that none of the dependencies in the input graphs are lost during the unification process. 

More formally, we prove the following:

\begin{theorem}
(Soundness of the unified model) For each causal execution dependency that the unified model expresses, the same causal execution dependency is expressed by one of the underlying partition models.
\end{theorem}

\begin{theorem}
(Completeness of the unified model) For each causal execution dependency expressed by any of the partition models, the same causal execution dependency is expressed by the unified model.
\end{theorem}

Due to space constraints, the proof of the algorithm's correctness is available in Appendix~\ref{sec:proof}. 
The proof confirms adherence to the requirements in section~\ref{sec:method}.

\subsection{Properties of the unification algorithms}

\label{sec:algorithm-properties}

The processing of the log to derive a U-CX graph involves multiple steps, with a total combined computational complexity of:

\begingroup
\vspace{-0.3cm}
\footnotesize
\[
O(T A + V A \log A + V + T A^3 + P A^2 + A^2 P^2 + A^2),\\
\]
\endgroup

where \( T \) is the number of traces in the log, \( A \) is the number of activities per trace, \( V \) represents the number of variants, and \( P \) is the number of partitions.
Each term corresponds to a step in the pipeline, mostly dominated by a cubic complexity in the maximal number of activities (i.e., $O(T A^3)$) in the employment of the CPD algorithm over the partitions. A detailed complexity breakdown is available in Appendix~\ref{app:algorithm-properties}.

As concluded in the evaluation section~\ref{sec:evaluation}, pragmatically, the computation time is not a major concern as in most realistic processes the number of activities in partitions is typically within a range of a few dozen.

\section{Evaluation}
\label{sec:evaluation}

We implemented and open sourced the method  at~~\textanon{\url{https://github.com/IBM/sax4bpm}\footnote{\url{https://doi.org/10.5281/zenodo.15539153}}}{\url{https://tinyurl.com/3ev7j8re}}.
We tested the scalability and performance of our algorithm against a handful of real datasets, three open and two from industrial applications:
\begin{itemize}
\vspace{-0.1cm}
    \item Road Traffic Fines (RTF): An event log for the process of managing road traffic fines by a local police force in Italy\footnote{\url{https://doi.org/10.4121/uuid:270fd440-1057-4fb9-89a9-b699b47990f5}}.
    \item Sepsis: An event log obtained from a regional hospital in The Netherlands\footnote{\url{https://doi.org/10.4121/uuid:915d2bfb-7e84-49ad-a286-dc35f063a460}}.
    \item BPIC12: BPI Challenge 2012 event log of a loan application process from a Dutch financial institution\footnote{\url{https://doi.org/10.4121/uuid:3926db30-f712-4394-aebc-75976070e91f}}.
    \item Helpdesk: A proprietary event log for a process specifically designed for managing support tickets, tracking issues, and assigning resolutions, issued by an IBM Process Mining (IPM) client company in Italy\footnote{\url{https://www.ibm.com/products/process-mining}}.
    \item CROMA: A proprietary event log for a medical equipment sterilization process captured by Croma Gio.Batta company in Spain\footnote{\url{https://www.cromagiobatta.it/en/home/}}.
\vspace{-0.2cm}
\end{itemize}

For each dataset, we measured partition-wise computation times and the total runtime for computing U-CXs across all partitions, as shown in Table~\ref{tab:benchmark-results}. We also assessed performance consistency with the algorithm properties outlined in section~\ref{sec:algorithm-properties}.

\begin{table}[htb]
\centering
\vspace{-0.3cm}
\caption{Benchmark results with three open datasets and two obtained from industry.}
\label{tab:benchmark-results}
\resizebox{\columnwidth}{!}{%
\begin{tabular}{llllllllllll}
\toprule
\textbf{Dataset} &
  \makecell{\textbf{\#} \\ \textbf{Traces}} &
  \makecell{\textbf{\#} \\ \textbf{Events}} &
  \makecell{\textbf{\#} \\ \textbf{Variants}} &
  \makecell{\textbf{Avg} \\ \textbf{event} \\ \textbf{/ trace}} &
  \makecell{\textbf{Min} \\ \textbf{event} \\ \textbf{/ trace}} &
  \makecell{\textbf{Max} \\ \textbf{event} \\ \textbf{/ trace}} &
  \makecell{\textbf{\#} \\ \textbf{Partitions}} &
  \makecell{\textbf{Avg time} \\ \textbf{/ partition} \\ {[}sec{]}} &
  \makecell{\textbf{Max time} \\ \textbf{/ partition} \\ {[}sec{]}} &
  \makecell{\textbf{Min time} \\ \textbf{/ partition} \\ {[}sec{]}} &
  \makecell{\textbf{Total} \\ \textbf{runtime} \\ {[}sec{]}} \\ 
\midrule
Road Traffic Fines & 150370 & 561470 & 231  & 4  & 2 & 20  & 35  & 0.121 & 0.147 & 0.003 & 22.04  \\
Sepsis             & 1050   & 15214  & 846  & 14 & 3 & 185 & 16  & 0.022 & 0.032 & 0.005 & 10.16  \\
BPIC12           & 13087  & 164506 & 4336 & 12 & 3 & 96  & 103 & 0.022 & 0.080 & 0.004 & 118.89 \\ 
\midrule
HELPDESK                & 3664   & 31581  & 972  & 8  & 1 & 47  & 51  & 0.017 & 0.028 & 0.001 & 14.13  \\
CROMA              & 595    & 7041   & 4    & 11 & 1 & 12  & 2   & 0.034 & 0.043 & 0.024 & 0.26  \\
\bottomrule
\end{tabular}%
}
\vspace{-0.6cm}
\end{table}

Figure \ref{fig:scatter-plots-comp-time} presents scatter plots for the BPIC12 and Helpdesk datasets, depicting the relationship between the number of activities and computation time across partitions. A cubic polynomial curve fits well, with $R^2=.9781$ for BPIC12 and $R^2=.9209$ for Helpdesk. Linear trend lines show similar fit levels ($R^2=.9708$ and $R^2=.9123$, respectively). Thus, scalability remains manageable in most practical cases, even with high trace counts, as seen in the Road Traffic Fines dataset. Overall, results align with expected performance properties.

\begingroup
\setlength{\tabcolsep}{4pt} 
\renewcommand{\arraystretch}{1.2} 
\begin{figure}[ht]
    \centering
    \begin{tabular}{c|c}
        \includegraphics[width=0.48\linewidth]{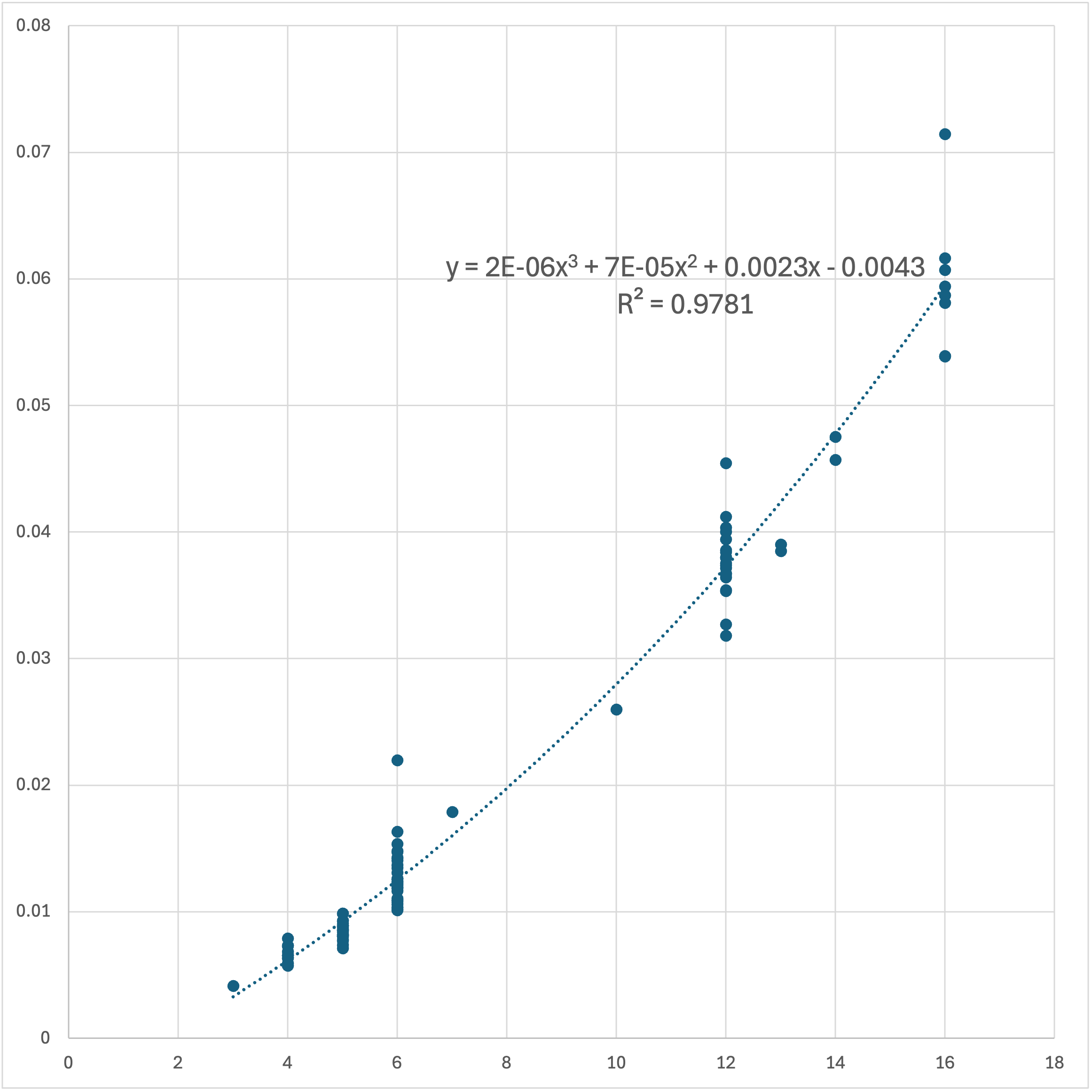} & \includegraphics[width=0.48\linewidth]{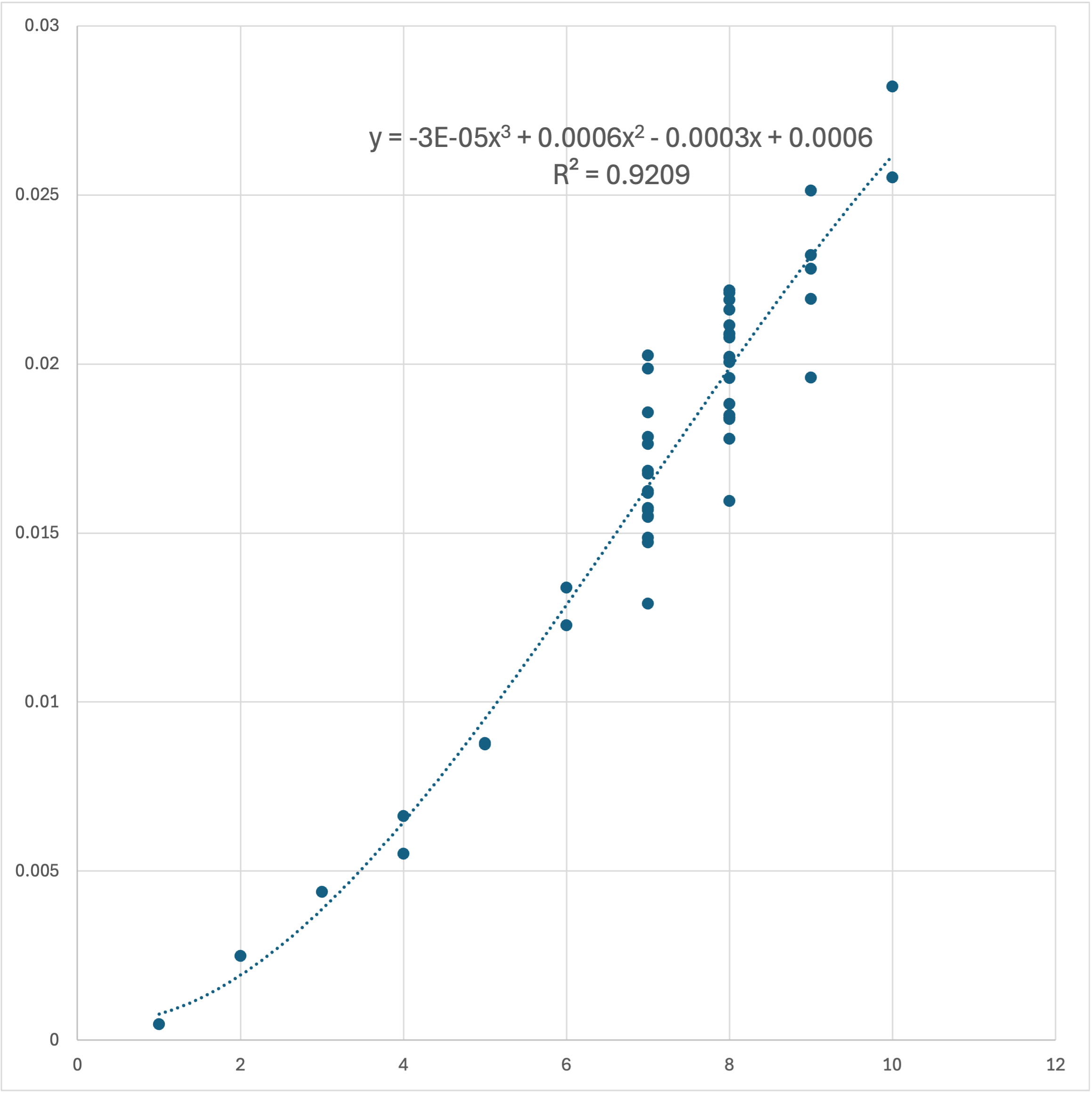}\\
        BPIC12 dataset & Helpdesk dataset
    \end{tabular}
    \vspace{-0.3cm}
    \caption{Partition compute times.}
    \vspace{-0.5cm}
    \label{fig:scatter-plots-comp-time}
\end{figure}
\endgroup

Figure~\ref{fig:bpi-2012-example} shows one example of two process variants from the BPIC12 event log, demonstrating the formation of a non-exhaustive $OR_C$ gateway and two $AND_C$ gateways.
Note that the explicit invocation set associated with the newly formed gateway ``or\_0'' is disclosed in the caption for the unified result.

\begingroup
\setlength{\tabcolsep}{2pt} 
\renewcommand{\arraystretch}{1} 
\begin{figure}[ht]
    \centering
    \begin{tabular}{c|c|c}
        \raisebox{0.51\height}
        {\includegraphics[width=0.27\linewidth]{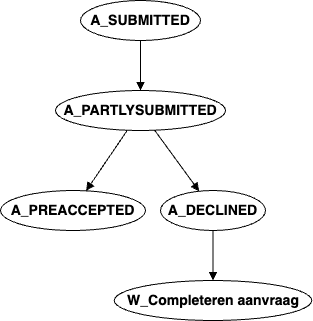}} & 
        \raisebox{0.49\height}
        {\includegraphics[width=0.36\linewidth]{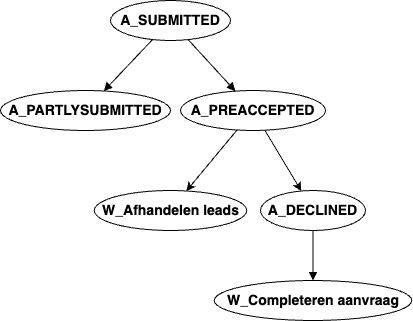}} & 
        \includegraphics[width=0.3\linewidth]{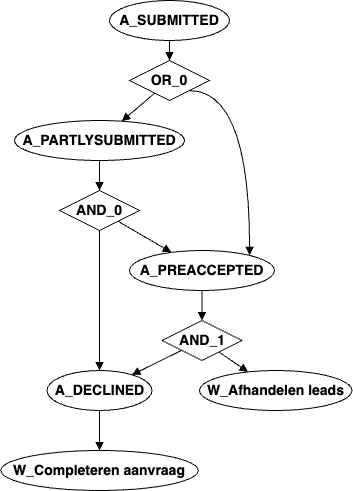}\\
        Variant 1: & Variant 2: & Unified result:\\
        \scriptsize
        \makecell{
        A\_SUBMITTED$\blacktriangleright$\\
        A\_PARTLYSUBMITTED$\blacktriangleright$\\
        A\_PREACCEPTED$\blacktriangleright$\\
        A\_DECLINED$\blacktriangleright$\\
        W\_Completeren aanvraag} &
        \scriptsize
        \makecell{
        A\_SUBMITTED$\blacktriangleright$\\
        A\_PARTLYSUBMITTED$\blacktriangleright$\\
        A\_PREACCEPTED$\blacktriangleright$\\
        W\_Afhandelen leads$\blacktriangleright$\\
        A\_DECLINED$\blacktriangleright$\\
        W\_Completeren aanvraag} & 
        \scriptsize
        \makecell{or\_0 denoting a\\ non-exhaustive gateway:\\
        \{(A\_PARTLYSUBMITTED,\\A\_PREACCEPTED),\\
        A\_PARTLYSUBMITTED
        \}
        }
    \end{tabular}
    \vspace{-0.1cm}
    \caption{BPIC12: unification of two example variants}
    \vspace{-0.5cm}
    \label{fig:bpi-2012-example}
\end{figure}
\endgroup

Figure~\ref{fig:ipm-screenshot} previews a new process mining feature under development in \textanon{the IBM}{a} process mining product, leveraging the developed model for prescriptive process analytics. This dashboard analyzes bottlenecks related to key performance indicators (KPIs) like process lead time, presenting historical KPI values alongside insights into activities contributing to delays while filtering out irrelevant ones. The U-CX graph enhances XAI by identifying critical activity execution times affecting the target KPI~\cite{Galanti2023AnAnalytics}, ranking activities along the causal chain, and excluding those with no impact.

\begin{figure}[ht]
    \centering
    \includegraphics[width=1\linewidth]{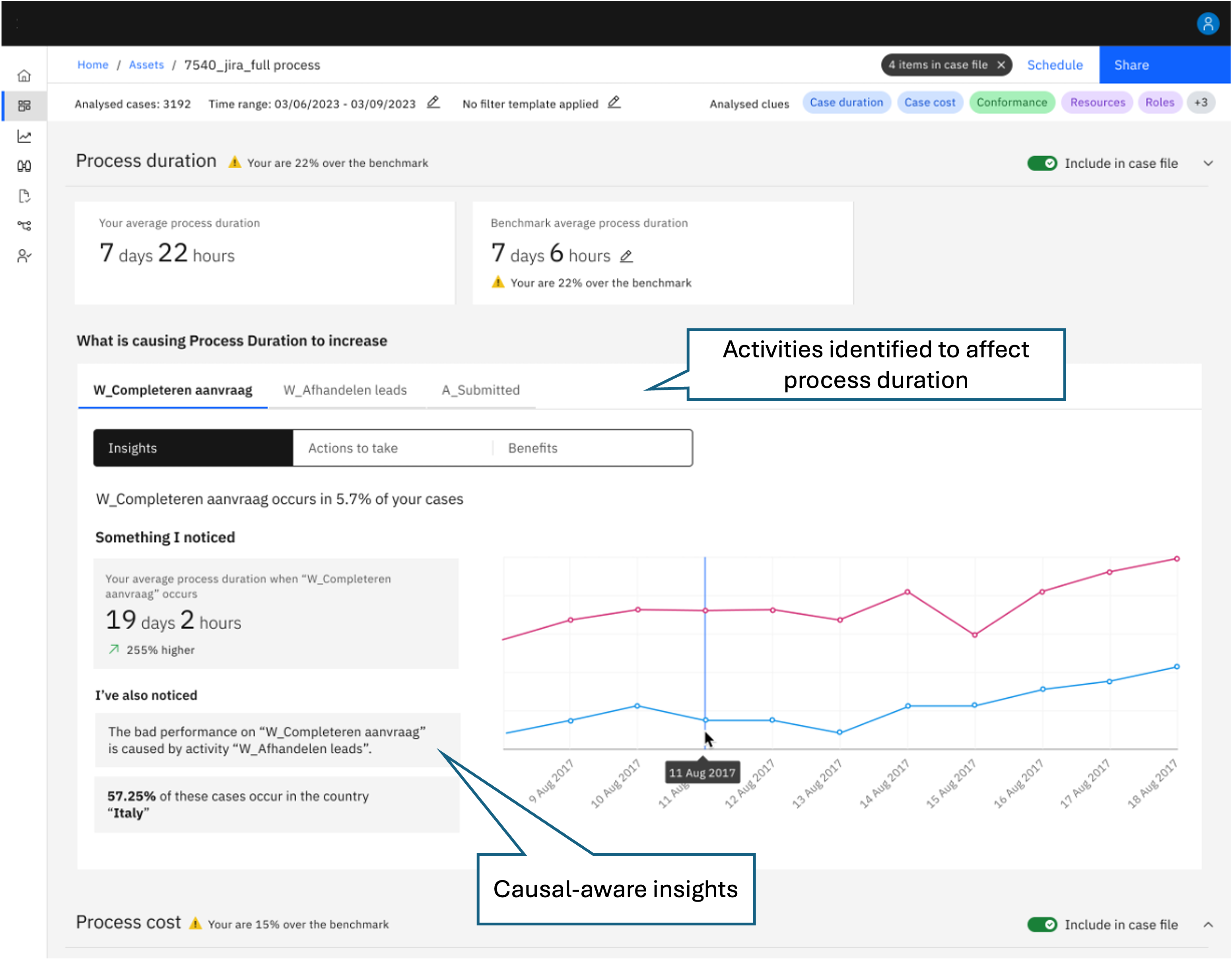}
    \vspace{-0.6cm}
    \caption{Screenshot of a prescriptive process analytics dashboard.}
    \vspace{-0.5cm}
    \label{fig:ipm-screenshot}
\end{figure}

With regards to the simplification step in our method, we attempted integration with an available solver\footnote{\url{https://github.com/schuyler/boolgen?tab=readme-ov-file}}, an implementation of the Quine–McCluskey algorithm (QMC)~\cite{Quine1955AFunctions}. As noted above, this realization helped in further compacting the logical expression. However, in certain cases, it also yielded negation and ``1''s as part of the concluded formulae, which is not expressible with our current graph notation. 

\section{Related Work}

Our work lies at the intersection of process discovery (PD) and causal discovery (CD). Addressing the multi-perspective paradigm in AI-Augmented BPMs~\cite{Dumas2023}, the model in~\cite{Fournier2023v3} is the first to systematically reveal the primary causal process structure from execution timestamps. Other causal perspectives, such as quality and cost, naturally derive from it~\cite{MichaelS.Dobson2004TheManagement}. This work further enhances the original model by incorporating ``causal-flow'' logic, ensuring it fully captures the altering causal conditions across process variants.

Most PD approaches apply some threshold on time precedence, counting `directly follows' or `eventually follows' relationships. ``Causal'' relations are sometimes used to express the frequency of these dependencies, as in~\cite{Kourani2023MiningGraphs}, which reduces representational bias in the hybrid miner algorithm by leveraging causal graph metrics for long-term dependencies. Instead, we take a bidirectional, asymmetric approach to analyzing timestamp relationships between activities.

Causal discovery aims to uncover causal relationships from observational data, distinguishing cause-effect directionality from mere correlation~\cite{Spirtes2001CausationSearch,Pearl2011Causality:Edition}. 
Methods divide into assessing intervention impact and identifying qualitative causal links. The Conditional Average Treatment Effect (CATE) technique gives a measurement to assess the magnitude of an intervention. It is used in \cite{Bozorgi2020} for decision-making and in~\cite{ShoushMahmoud2022WhenConstraints} for prescriptive process monitoring. While these studies assume a given causal model, we focus on discovering causal models.

In~\cite{Hompes2017DiscoveringVariation}, a graph of causal factors is generated using Granger causality, a statistical test for time-series analysis. However, relations are based on timestamped KPI values, whereas we analyze activity timestamps.

Prior work has explored causal relationships among decision points~\cite{Alaee2024Data-DrivenAndDiscovery,Tanmayee2019,Leemans2022CausalModels}. Our method, however, uncovers how decisions, represented by causal gateways, influence subsequent activity executions.
Causal discovery algorithms face challenges like handling latent variables, feedback loops, and missing data. Methods such as multiple imputation \cite{Rubin1987MultipleSurveys}, expectation-maximization (EM) \cite{Dempster1977Algorithm}, and inverse probability weighting \cite{Seaman2013ReviewData} address missing data but have limitations in process mining. Multiple imputations risk bias if misspecified, and EM algorithms are computationally demanding for large datasets~\cite{Mohan2013MissingProblem}. These approaches also rely on strong assumptions about data distribution and missingness. Instead, by splitting data into subsets, we avoid the need for imputation or complex modeling. The gating mechanism integrates individual causal graphs, ensuring the unified model captures the causal relationships across subsets.

Causal discovery also employs some modeling language for result articulation. C-Nets~\cite{vanderAalst2016ProcessMining,vanderAalst2011CausalDiscovery}, a rich process notation compared to others (e.g., Petri-nets, BPMN, and EPC), could serve as a viable alternative for U-CX graphs. By using input/output bindings, C-Nets capture diverse execution sequences. However, as a notation, they lack a mechanism to infer causal structures, requiring expert input or a unification method like the one proposed in our work.

\section{Conclusions and future work}


We introduce a method for integrating multiple causal graphs into a unified model, preserving individual causal dependencies while explicitly capturing alternating execution conditions via a gating mechanism. Unlike conventional \textit{always causes} approaches, our method adopts a \textit{sometimes causes} perspective, adapting arrow notation accordingly and extending it with graphical symbols for temporal alternations. We formalized and implemented the method, analyzed its complexity, proved its correctness, and evaluated its scalability using five benchmark datasets—three open and two proprietary. 


The method offers key advantages. It scales well for large process mining datasets by operating on subsets. While extending a prior causal process discovery technique, it remains agnostic to the specific causal discovery algorithm used. By handling missing data through splitting and unification, our method mitigates the risk of introducing bias due to incorrect assumptions about missing values.

Our diamond-based gateway notation enhances visualization in a BPMN-like manner, distinguishing logical alternations (AND, OR, XOR). However, it does not explicitly convey edge coefficient values derived from the causal discovery algorithm. Users seeking these values must refer to the input graphs or rely on annotations in the extended output. While numeric figures can be added next to alternatives (e.g., the $OR_C$ gateway in Figure~\ref{fig:four-gateways}), such annotations can become cumbersome as the number of alternatives per gateway increases.


Our model could serve as a semantic foundation for C-Nets, where each node has inbound and outbound sets represented by dots connected by arcs. Extending this notation with numeric indicators for coefficient values could enhance clarity but may also lead to dense visualizations with multiple arcs, affecting usability. Furthermore, incorporating additional literals from the simplification step (e.g., negation, ``1''s) remains an open question. Future work could explore leveraging large language models for simplification using few-shot learning and user-rated outputs to balance compactness and completeness.




\section*{Acknowledgments}
\textanon{This project has received funding from the European Union’s Horizon research and innovation programme under grant agreements no 101094905 (AI4GOV), 101092021 (AutoTwin), and 101092639 (FAME). We thank Croma Gio.Batta and IBM Process Mining for sharing their datasets.}{}

\bibliographystyle{splncs04-no-url-doi}
\bibliography{referencesm} 

\newpage
\appendix
\section*{Appendices}
\include{gateways}
\include{log-split-alg}
\include{matrix-construction}
\include{unification-alg}
\include{matrix-to-graph}
\include{properties-of-alg}
\include{proof}

\end{document}

%% file: gateways.tex
\section{Causal Gateways Specification}
\label{sec:gateways}
\begin{definition}
A \textit{Causal "And" Gateway}, denoted as \( \text{AND}_C \), is a node in a causal execution graph \( G = (V, E) \), where \( \text{AND}_C \in V \), that links a source edge \( (s, \text{AND}_C) \in E \) and a set of two or more target edges \( \{(\text{AND}_C, t) \mid t \in T_e, T_e \subseteq V, |T_e| \geq 2 \} \subseteq E \). The gateway \( \text{AND}_C \) denotes that the node \( s \in V \) linked via the source edge \( (s, \text{AND}_C) \) causes the execution of all nodes in \( T_e \) linked via the set of target edges.
\end{definition}

\begin{definition}
A \textit{Causal "Or" Gateway}, denoted as \( \text{OR}_C \), is a node in a causal execution graph \( G = (V, E) \), where \( \text{OR}_C \in V \), that links a source edge \( (s, \text{OR}_C) \in E \) and a set of two or more target edges \( \{(\text{OR}_C, t) \mid t \in T_e, T_e \subseteq V, |T_e| \geq 2 \} \subseteq E \). The gateway \( \text{OR}_C \) denotes that the node \( s \in V \) linked via the source edge \( (s, \text{OR}_C) \) causes the execution of some (at least one) of the nodes in \( T_e \) linked via the set of target edges.
\end{definition}

Concerning the general type of the "Or" gateways, we also define a specific type in which the source activity can be a cause for any possible combination among its target activities as follows:

\begin{definition}
A \textit{Exhaustive Causal "Or" Gateway}, denoted as \( \text{OR}_C^E \), is a special type of causal "Or" gateway in a causal execution graph \( G = (V, E) \), where \( \text{OR}_C^E \in V \), that links a source edge \( (s, \text{OR}_C^E) \in E \) and a set of two or more target edges \( \{(\text{OR}_C^E, t) \mid t \in T_e, T_e \subseteq V, |T_e| \geq 2 \} \subseteq E \). The gateway \( \text{OR}_C^E \) denotes that the node \( s \in V \) linked via the source edge \( (s, \text{OR}_C) \) causes any combination of the nodes in \( T_e \) linked via the set of target edges, including one or multiple target nodes.
\end{definition}

\begin{definition}
A \textit{Causal "Xor" Gateway}, denoted as \( \text{XOR}_C \), is a node in a causal execution graph \( G = (V, E) \), where \( \text{XOR}_C \in V \), that links a source edge \( (s, \text{XOR}_C) \in E \) and a set of two or more target edges \( \{(\text{XOR}_C, t) \mid t \in T_e, T_e \subseteq V, |T_e| \geq 2 \} \subseteq E \). The gateway \( \text{XOR}_C \) denotes that the node \( s \in V \) linked via the source edge \( (s, \text{XOR}_C) \) causes the execution of exactly one node in \( T_e \) linked via the set of target edges.
\end{definition}

%% file: log-split-alg.tex
\section{Log Partitioning Algorithm}
\label{app:log-partitioning}

\begin{algorithm}[!ht]
\footnotesize
\caption{Log Partitioning}
\textbf{Log Partitioning Function:} \\
\textbf{Input:}
\begin{itemize}
\vspace{-0.2cm}
    \item Event log \( L \) (a set of traces)
    \item \textit{Optional:} Selection of variants \( V_s \), each identifying a subset of traces in \( L \); default: \( V_s = \emptyset \)
    \item \textit{Optional:} Boolean flag \texttt{split\_by\_variants}, default: \texttt{false}
\vspace{-0.2cm}
\end{itemize}

\textbf{Output:} A set of causal execution graphs \( G_i \); initialized as \( G_i = \emptyset \)

\begin{algorithmic}[1]
\If{\( V_s = \emptyset \)}
    \State \( V_s \gets \text{PM.getAllVariants}(L) \) \Comment{Retrieve all variants using a process mining algorithm}
\EndIf
\State \( T_s \gets \bigcup V_s \) \Comment{Assign \( T_s \) as the union of all traces in the variants \( V_s \)}
\If{\texttt{split\_by\_variants} = \texttt{true}}
    \State \( P_s \gets V_s \) \Comment{Directly assign \( V_s \) to \( P_s \) if splitting by variants}
\Else
    \State \( P_s \gets \text{SplitFunction}(T_s) \) \Comment{Invoke the \texttt{SplitFunction} on \( T_s \) to get partitions}
\EndIf
\For{\textbf{each} partition \( p_i \in P_s \)}
    \State \( g_i \gets \text{CBP.discover}(p_i) \) \Comment{Derive the causal execution graph for \( p \)}
    \State \( G_i \gets G_i \cup \{g_i\} \) \Comment{Add the element \( g_i \) to the set \( G_i \)}
\EndFor
\State \textbf{return} \( G_i \)
\end{algorithmic}
\vspace{0.2cm}
\textbf{Split Function:} \\
\textbf{Input:} A set of traces \( T \) \\
\textbf{Output:} A set of partitions \( P \), where each partition contains traces from \( T \) that share the same set of activities (regardless of order).

\begin{algorithmic}[1]
\State Initialize \( P \gets \emptyset \) \Comment{The set of partitions}
\State Initialize a mapping \( M \gets \emptyset \) \Comment{Maps sets of activities to partitions}
\For{\textbf{each} trace \( t \in T \)}
    \State \( A_t \gets \text{Set}(t) \) \Comment{Compute the set of activities in \( t \)}
    \If{\( A_t \notin M \)}
        \State \( M[A_t] \gets \emptyset \) \Comment{Create a new partition for this set of activities}
    \EndIf
    \State \( M[A_t] \gets M[A_t] \cup \{t\} \) \Comment{Add the trace \( t \) to the appropriate partition}
\EndFor
\State \( P \gets \{ M[k] \mid k \in M \} \) \Comment{Collect all partitions from the mapping \( M \)}
\State \textbf{return} \( P \)
\end{algorithmic}
\label{alg:log-partitioning}
\end{algorithm}

%% file: matrix-construction.tex
\section{Input Graphs to Matrix Representation Algorithm}
\label{sec:matrix-construction}

\begin{algorithm}[!ht]
\footnotesize
\caption{Processing Input Graphs into a Matrix Representation}
\label{alg:process_graphs}
\begin{algorithmic}[1]
\Require Set of graphs $\mathcal{G_i} = \{g_1, g_2, \dots, g_n\}$, where $g_i = (N_i, E_i)$
\Ensure Matrix $M[node][graph]$ representing child sets

\State Compute the union of nodes: $N \gets \bigcup_{i=1}^n N_i$

\State Initialize the matrix $M$:
\ForAll{node $u \in N$ and graph $g_i \in \mathcal{G_i}$}
    \State $M[u][g_i] \gets \emptyset$
\EndFor

\State Populate the matrix:
\ForAll{graph $g_i = (N_i, E_i) \in \mathcal{G_i}$}
    \ForAll{edge $(u, v) \in E_i$}
        \State $M[u][g_i] \gets M[u][g_i] \cup \{v\}$
    \EndFor
\EndFor

\State \Return $M$
\end{algorithmic}
\end{algorithm}

%% file: unification-alg.tex
\section{Unification Algorithm}
\label{sec:unification-alg}
\begin{algorithm}[!ht]
\footnotesize
\caption{Processing the Matrix with the Unification Algorithm}
\begin{algorithmic}[1]
\Require Matrix $M[node][graph]$, where columns correspond to graphs \(g_1, g_2, \dots, g_n\)
\Ensure Revised matrix $M'$ with gateway markings and viable alternatives for all $OR_C$ gateways in $M_{OR}$ 

\State $M' \gets M$

\State \( M_{OR} \gets \emptyset \) \Comment{Initialize map sets to capture alternatives executions for $OR_C$}

\For{each node $u \in M$}
    \State $\text{Family} \gets \{M[u][g] \mid g \in \{g_1, g_2, \dots, g_n\}, M[u][g] \neq \emptyset\}$

    \Comment{Step 1: $AND_C$ Identification}
    \For{each child set $S_i \in \text{Family}$}
        \State $\text{IntersectsPartially} \gets \text{False}$
        \For{each child set $S_j \in \text{Family}$, $j \neq i$}
            \If{$S_i \cap S_j \neq \emptyset$ \textbf{and} $S_i - S_j \neq \emptyset$} \Comment{Checking no partial intersection}
                \State $\text{IntersectsPartially} \gets \text{True}$
                \State \textbf{break}
            \EndIf
        \EndFor
        \If{\textbf{not} $\text{IntersectsPartially}$}
            \For{each occurrence of $S_i$ in $\text{Family}$}
                \State $S_i \gets "(" + S_i + ")"$ \Comment{Mark child set as $AND_C$}
            \EndFor
        \EndIf
    \EndFor

    \Comment{Step 2: $XOR_C$ Check (only if row has 2+ sets)}
    \If{$|\text{Family}| \geq 2$ \textbf{and} $\forall S_i, S_j \in \text{Family}, i \neq j, S_i \cap S_j = \emptyset$} \Comment{Checking family is an exclusive set}
        \State $u \gets "\}" + u + "\{"$ \Comment{Mark node as $XOR_C$}
    \Else
        \Comment{Step 3: $OR_C^E$ Check (only if row has 2+ sets)}
        \If{$|\text{Family}| \geq 2$}
            \State $U \gets \bigcup_{S \in \text{Family}} S$ \Comment{Compute the union of all sets in Family}
            \If{$|\text{Family}| = 2^{|U|}$} \Comment{Check if Family is a powerset of \(U\)}
                \State $u \gets "[" + u + "]"$ \Comment{Mark node as $OR_C^E$}
            \Else
                \State $u \gets "*" + u + "*"$ \Comment{Default to $OR_C$}
                \State \( M_{OR}[u] \gets \emptyset \) \Comment{Create a new map entry for this family set}
                \State \( M_{OR}[u] \gets M_{OR}[u] \cup \text{Family} \) \Comment{Add the Family set to the entry}
            \EndIf
        \EndIf
    \EndIf

\EndFor

\State \Return $M'$, $M_{OR}$

\end{algorithmic}
\label{alg:unification-algorithm}
\end{algorithm}

%% file: matrix-to-graph.tex
\section{Unified Graph Construction Algorithm}
\label{sec:matrix-to-graph}
\vspace{-0.5cm}
\begin{algorithm}[!ht]
\footnotesize
\caption{Construct Unified Graph from Revised Matrix}
\begin{algorithmic}[1]
\Require Revised matrix \(M'\), where each row represents a node and columns represent its child sets
\Ensure Unified graph \(G_u = (N, E)\), with \(N\) as nodes and \(E\) as edges

\State Initialize \(N \gets \emptyset\) and \(E \gets \emptyset\) \Comment{Create an empty graph}

\For{each row \(u \in M'\) with row number \(r\)} \Comment{Iterate through nodes in the revised matrix}
    \State \(N \gets N \cup \{u\}\) \Comment{Add node \(u\) to \(N\)}

    \For{each child set \(S \in M'[u]\)}
        \If{\(S = \emptyset\)} \Comment{Skip empty sets}
            \State \textbf{continue}
        \EndIf
        \If{\(|S| = 1\)} \Comment{\(S\) contains a single element}
            \State \(N \gets N \cup \{v\}\) \Comment{Add node \(v\) to \(N\)}
            \State \(E \gets E \cup \{(u, v)\}\) \Comment{Add edge from \(u\) to \(v\)}
        \ElsIf{\(S = \{(v_1, v_2, \dots, v_k)\}\) \textbf{and} \(label(S) = "(...)"\)} 
        \State \Comment{\(S\) contains multiple elements and is surrounded by round brackets}
            \State \(N \gets N \cup \{\text{AND}_{Cr}\}\) \Comment{Add \(\text{AND}_{Cr}\) to \(N\)}
            \State \(E \gets E \cup \{(u, \text{AND}_{Cr})\}\) \Comment{Add edge from \(u\) to \(\text{AND}_{Cr}\)}
            \For{each \(v_j \in \{v_1, v_2, \dots, v_k\}\)}
                \State \(E \gets E \cup \{(\text{AND}_{Cr}, v_j)\}\) \Comment{Add edge from \(\text{AND}_{Cr}\) to \(v_j\)}
            \EndFor
        \EndIf
    \EndFor

    \If{\(label(u) = "\}...\{"\)} \Comment{Handle $XOR_C$}
        \State \(N \gets N \cup \{\text{XOR}_{Cr}\}\) \Comment{Add \(\text{XOR}_{Cr}\) to \(N\)}
        \For{each edge \((u, v) \in E \text{ where } v \text{ is a child of } u\)}
            \State \(E \gets E \setminus \{(u, v)\}\) \Comment{Remove edge \((u, v)\)}
            \State \(E \gets E \cup \{(u, \text{XOR}_{Cr}), (\text{XOR}_{Cr}, v)\}\) \Comment{Redirect edges via \(\text{XOR}_{Cr}\)}
        \EndFor
    \ElsIf{\(label(u) = "[...]"\)} \Comment{Handle $OR_C^E$}
        \State \(N \gets N \cup \{\text{OR}_{Cr}^E\}\) \Comment{Add \(\text{OR}_{Cr}^E\) to \(N\)}
        \For{each edge \((u, v) \in E \text{ where } v \text{ is a child of } u\)}
            \State \(E \gets E \setminus \{(u, v)\}\) \Comment{Remove edge \((u, v)\)}
            \State \(E \gets E \cup \{(u, \text{OR}_{Cr}^E), (\text{OR}_{Cr}^E, v)\}\) \Comment{Redirect edges via \(\text{OR}_{Cr}^E\)}
        \EndFor
    \ElsIf{\(label(u) = "*...*"\)} \Comment{Handle $OR_C$}
        \State \(N \gets N \cup \{\text{OR}_{Cr}\}\) \Comment{Add \(\text{OR}_{Cr}\) to \(N\)}
        \For{each edge \((u, v) \in E \text{ where } v \text{ is a child of } u\)}
            \State \(E \gets E \setminus \{(u, v)\}\) \Comment{Remove edge \((u, v)\)}
            \State \(E \gets E \cup \{(u, \text{OR}_{Cr}), (\text{OR}_{Cr}, v)\}\) \Comment{Redirect edges via \(\text{OR}_{Cr}\)}
        \EndFor
    \EndIf
\EndFor

\State \Return \(G_u = (N, E)\)
\end{algorithmic}
\label{alg:graph-reconstruction}
\end{algorithm}

%% file: properties-of-alg.tex
\section{Properties of the Algorithms}
\label{app:algorithm-properties}

The processing of the log to derive a U-CX graph consists of several steps. The first step, \textbf{variant identification}, involves splitting the log into variants by grouping traces that follow the same sequence of activities. This is accomplished using a hashing approach that stores string representations of activity sequences as hash keys, resulting in an overall computational complexity of \( O(TA) \), where \( T \) is the number of traces in the log and \( A \) is the number of activities per trace.

Following the identification of variants, the process proceeds to \textbf{variant ordering}, which ensures that variants with the same set of activities, irrespective of order, are grouped together. This requires sorting the activity sequences of the variants, leading to a complexity of \( O(VA \log A) \), where \( V \) represents the number of unique variants.

Once ordered, the next step is \textbf{partition identification}, where variants with identical sets of activities are combined into partitions. With pre-sorted variants, this operation can be completed in linear time with a complexity of \( O(V) \).

After partitioning, the process moves to \textbf{causal discovery}, where the LiNGAM algorithm is applied to the traces within each partition to infer causal dependencies among activities. Given the cubic complexity of LiNGAM with respect to the number of activities, the overall complexity for this step is \( O(T A^3) \).

The resulting causal graphs are then transformed into a \textbf{matrix representation} in the graph transformation step. Each row in this matrix corresponds to a parent activity node in the causal graphs, and columns contain sets of child activities that are causally related to the parent. The complexity of this step is \( O(P A^2) \), where \( P \) is the number of partitions.

Next, the matrix undergoes processing in the \textbf{matrix unification} step to identify the causal gateways. This step involves scanning the matrix to detect partial intersections and exclusivity conditions. Checking for partial intersections and exclusivity, which dominate the computational cost, results in a complexity of \( O(A^2 P^2) \), while the complexity of checking for the exhaustive OR involves a simple count of the child sets and size comparison \( O(1) \) using the formula:
\begingroup
\footnotesize
\[
|S| = 2^{|\cup S|}
\]
\endgroup

where $S$ is the union set of activities in all child sets. Finally, the marked matrix is transformed back into a U-CX graph in the \textbf{graph reconstruction} step. This involves processing each row of the matrix to reconstruct nodes and edges, yielding a complexity of \( O(A^2) \).

Combining the complexities of all steps, the total computational complexity of the entire process is:

\begingroup
\footnotesize
\[
O(T A + V A \log A + V + T A^3 + P A^2 + A^2 P^2 + A^2),\\
\]
\endgroup

where each term corresponds to a step in the pipeline, mostly dominated by a cubic complexity in the maximal number of activities (i.e., $O(T A^3)$) in the employment of LiNGAM over the partitions.

%% file: proof.tex
\section{Proof}
\label{sec:proof}
In this section, we present a proof for the soundness and completeness properties of our unification algorithm. Given a set of input causal execution graphs \( G_i = [G_1, \dots, G_n]\) and a unified, extended, causal execution output graph \( G_U \), we define soundness and completeness as follows:

\paragraph{Soundness:}
A unified causal execution graph \( G_U \) is \textbf{sound} if every causal dependency that can be derived in \( G_U \) must have originated from at least one of the input graphs in \( G_i \).  

In other words, soundness ensures that the unification process does \textbf{not introduce any new causal dependencies} that were not already present in one of the input graphs. This guarantees that \( G_U \) remains a \textbf{valid representation} of the input graphs and does not create incorrect causal dependencies.

\paragraph{Completeness:}
A unified causal execution graph \( G_U \) is \textbf{complete} if every causal dependency that appears in at least one input graph must also be derivable in \( G_U \).  

This means that no edges from the original input graphs are \textbf{lost} during the unification process. Completeness ensures that \( G_U \) \textbf{fully preserves} the structural relationships found in the input graphs and does not omit any existing information.

\begin{lemma}
The models $G_i$ corresponding to each partition are each both sound and complete w.r.t to the partition traces.
\end{lemma}

\begin{theorem}
(Soundness of the unified model) For each causal execution dependency that the unified model expresses, the same causal execution dependency is expressed by one of the underlying partition models.
\end{theorem}

\begin{proof}


We define a path \( p(n_s, n_t) \) between a source node \( n_s \) and a target node \( n_t \) as either a single direct edge \( (n_s, n_t) \) or a gateway-mediated path, consisting of a pair of edges \( (e_s, e_t) \), where \( e_s = (n_s, GATE) \) and \( e_t = (GATE, n_t) \). The node \( GATE \) represents one of the causal gateway types: \( AND_C \), \( OR_C \), \( OR_C^E \), or \( XOR_C \), whether split or join. Next, our formalism corresponds to a split gateway type, without the loss of generality, the proof is the same for all join gateway, with the difference of having everything mirrored in its direction.

To prove the theorem, we must show the following.

For every node \( n_j \) in \( G_U \), we consider a set of target nodes $ N_t = \{ n_{t1}, \dots, n_{tk} \}$ $\subseteq G_U.nodes $, which represents the potential target nodes reachable from \( n_j \). 

The type of \( GATE \) determines how these target nodes are grouped into subsets. As a preceding step, if the type is an $AND_C$ gateway, we replace the group of all corresponding target nodes in $N_t$ with a single literal as a `syntactic sugar' for that group. Subsequently, if the type is a $XOR_C$, then the partitioning is an exclusive split of $N_t$. If the type is a $OR_C^E$, then the partitioning is the powerset of $N_t$. For a type of a non-exhaustive $OR_C$, the partitioning is some partial set of the powerset that can only be resolved if explicitly specified.

Let \( N_{\hat{t}_k} \subseteq N_t \) be a subset of these target nodes forming a \textbf{family of sets} indexed by \( k \), where each \( N_{\hat{t}_k} \) represents a distinct path group.

The corresponding family of paths is given by:
\[
P = \{ P_k \mid P_k = \{ p(n_j, n_x) \mid n_x \in N_{\hat{t}_k} \}, k \in K \}
\]
and the associated family of edge sets is:
\[
E = \{ E_k \mid E_k = \{ (n_j, n_x) \mid n_x \in N_{\hat{t}_k} \}, k \in K \}.
\]

If each path group \( P_k \) is determined by the behavior of \( GATE \) type, then there exists a graph \( g_i \in G_i \) such that each edge set \( E_k \) is exactly the set of edges in \( g_i \) that originate from \( n_j \), meaning:
\[
E_k = \{ e \in g_i.edges \mid e = (n_j, n_x) \text{ for some } n_x \in g_i.nodes \}.
\]

In other words, the set \( E_k \) must contain \emph{all and only} the edges originating from \( n_j \) in \( g_i \). This ensures that the outgoing edges from \( n_j \) in \( G_U \) are fully determined by a single graph \( g_i \).



For the case where the set of edges $E$ is a single direct edge $\{(n_j,n_x)\}$, we need to show that there exists a graph $g_i \in G_i$ where $E = \{ e \in g_i.edges \mid e = (n_j, n_x) \text{ for some } n_x \in g_i.nodes \}$. An edge $(n_j, n_x)$ is inferred in $G_U$ according to our algorithm when $n_x$ is a single descendant of $n_j$ in at least one or more $g_i \in G_i$, i.e., the family set corresponding to $n_j$ in $M[n_j][*]$ contains either elements equal to $\{ n_x \}$ or empty sets $\{ \emptyset \}$. Thus, satisfying the condition.

For the case where the set of edges $E$ includes more than a single edge, corresponding to a family of gateway-mediate paths, we prove the theorem by the case of each specific gateway type.


For the case of an \( XOR_C \) gateway, we establish how the original set of target nodes in $N_t$ was formed in the algorithm.
Any edge \( (n_j, n_x) \) in \( E_k \) was inferred as part of a path \( p(n_j, n_x) \) that includes an \( XOR_C \) gateway. Specifically, our algorithm labeled the row \( M[n_j][*] \) as \( XOR_C \) if and only if \( n_x \) was part of an exclusive family set in a specific cell $M[n_j][g_i]$, and \( n_x \) was included in that cell only if it was originally identified as a descendant of \( n_j \) in some \( g_i \in G_i \). 

Thus, the condition for an \( XOR_C \) gateway is satisfied.

For the case of an \( OR_C^E \) gateway, we establish how the original set of target nodes in \( N_t \) was formed in the algorithm. Any edge \((n_j, n_x)\) in \( E_k \) was inferred as part of a path \( p(n_j, n_x) \) that includes an \( OR_C^E \) gateway. Specifically, our algorithm labeled the row \( M[n_j][*] \) as \( OR_C^E \) if and only if that row was equal to \( \mathcal{P} \Bigl(\bigcup_{g_i \in G_i} M[n_j][g_i] \Bigr) \), where \(\mathcal{P}(\cdot)\) denotes the \textbf{power set} operator.

Moreover, \((n_j, n_x)\) appears in that row only if \( n_x \) was originally identified as a descendant of \( n_j \) in some \( g_i \in G_i \), and since \( n_x \) is included in the union, it must also belong to the power set. Consequently, each \( E_k \) has a corresponding cell in the row \( M[n_j][g_i] \), indicating that \( n_x \) was indeed part of the power set of all possible children of \( n_j \).

Thus, the condition for an \( OR_C^E \) gateway is satisfied: we obtain every possible combination of descendants for \( n_j \) in the power set, ensuring \((n_j, n_x)\) was drawn from at least one graph \( g_i \), thereby enforcing the exhaustive nature of the \( OR_C^E \) gateway.

For the case of a non-exhaustive \( OR_C \) gateway, we establish how the original set of target nodes in \( N_t \) was formed in the algorithm. Unlike the \( OR_C^E \) case, where the row in \( M[n_j][*] \) corresponds to the power set of possible descendants, a non-exhaustive \( OR_C \) gateway represents a \textbf{subset} of the power set. That is, the annotated row in the matrix satisfies: \( M[n_j][*] \subset \mathcal{P}(\cdot) \).

Without an explicit specification of the target node sets in \( N_t \), it is \textbf{not possible} to determine whether each edge set \( E_k \) has a corresponding graph \( g_i \). This is because the partial subset selection means that some combinations of descendants may not correspond directly to any single input graph.

However, if an explicit specification of the sets in \( N_t \) is disclosed, it resolves the association of each \( E_k \) to a corresponding \( g_i \). That is, once the partial selection is defined, we can establish which elements of the power set correspond to edges derived from some subgraph \( g_i \). Consequently, each \( E_k \) has a corresponding cell in the row \( M[n_j][g_i] \), ensuring that each edge \( (n_j, n_x) \) is accounted for within some input graph.

Thus, the condition for a non-exhaustive \( OR_C \) gateway is satisfied: the inferred edges result from a partial selection of possible descendant combinations, whose explicit specification allows the correct association of each \( E_k \) with at least one graph \( g_i \).

\end{proof}

\begin{theorem}
(Completeness of the unified model) For each causal execution dependency expressed by any of the partition models, the same causal execution dependency is expressed by the unified model.
\end{theorem}

\begin{proof}


For every node \( n_j \in g_i.nodes \), where \( g_i \in G_i \), we consider a set of target nodes \( N_t = \{ n_{t1}, \dots, n_{tk} \} \subseteq g_i.nodes \), where \( (n_j, n_{tk}) \in g_i.edges \). We prove that \( G_U \) is complete by showing that for every target node \( n \in N_t \), either \( (n_j, n) \in G_U.edges \) or there exists a path \( p(n_j, n) \) in \( G_U \).

With respect to the set of targets \( N_t \) corresponding to some \( g_i \), the following cases are considered.

The first case is when the set of targets includes only a single element \( n_t \). This occurs either when \( n_t \) is a target only in a single \( g_i \in G_i \) and does not appear as a target in any other \( g_j \in G_i \) where \( g_i \neq g_j \), or when \( n_t \) is a target in multiple graphs \( G_k \subset [G_1, \dots, G_n] \), but the target sets of all remaining graphs \( G_k \) are empty.

According to our algorithm, a single-element target in a single \( g_i \) would correspond to a row \( M[n_j][*] \) with a column for \( g_i \) containing ${n_t}$ and all remaining columns for all remaining graphs in \(G_i\setminus g_i \) containing $\emptyset$. Such a row would not have been annotated. As a result, the direct edge \( (n_j, n_t) \) is included in \( G_U.edges \).

If \( n_t \) appears as a target in multiple graphs \( G_k \subset [G_1, \dots, G_n] \), but no other target nodes exist in the remaining graphs, then the corresponding row \( M[n_j][*] \) will have multiple cells containing \( n_t \) and all remaining cells corresponding to \(G_i\setminus G_k \) containing $\emptyset$. Since such a row would not have been annotated in the algorithm, this also entails that the direct edge \( (n_j, n_t) \) is included in \( G_U.edges \).

We now extend the analysis to cases where the set of target nodes \( N_t \) contains multiple elements. This entails two subsequent steps.

First, if the set of elements in \( N_t \) corresponding to \( g_i \) has no partial intersection with any of the other target sets, this would correspond to a row \( M[n_j][*] \) with a cell containing \( N_t \) that has no intersection with the sets in the other cells. Such a row would be annotated as 'AND\_C', thus including each node \( n \in N_t \) in edges in \( G_U \) of the form \( (AND_C, n) \). As in the previous case, for the next step, we consider such sets to be replaced by a single literal $n$ as a syntactic sugar.

Thus, as a second step, a target set of elements in \( N_t \) corresponding to \( g_i \) may be an element of an exclusive set of sets (i.e., the targets of all graphs in \( G_i \)). In this case, according to the algorithm, the cells in the row \( M[n_j][*] \) would have been annotated as $XOR_C$, and as a result all nodes \( n \in N_t \) would have been added in edges in \( G_U \) of the form \( (XOR_C, n) \), with the edge \( (n_j, XOR_C) \) also added to \( G_U \). This would thus entail the inclusion of the path \( p(n_j, n)=\{(n_j,XOR_C),(XOR_C,n)\} \) in \( G_U \). If the node $n$ denotes a syntactic sugar for the case of an $AND_C$ above, the path extends further to the set \(\{(n_j,XOR_C),\) \( (XOR_C,AND_c),\) \( (AND_c,n)\} \) in \( G_U \). 

If the set of elements in \( N_t \) corresponding to \( g_i \) is not an element of an exclusive set of sets (i.e., the targets of all graphs in \( G_i \)), it may be an element in the powerset of the union of all target sets of all graphs in \( G_i \). In this case, according to our algorithm, the cells in the row \( M[n_j][*] \) would have been annotated as $OR_C^E$, and all nodes \( n \in N_t \) would have been added in edges in \( G_U \) of the form \( (OR_C^E, n) \), with the edge \( (n_j, OR_C^E) \) also added to \( G_U \), thus entailing the inclusion of the path \( p(n_j, n) \) in \( G_U \).

Eventually, considering the above two alternatives are not met and the set of elements in \( N_t \) corresponding to \( g_i \) is neither part of an exclusive set, nor a part of a powerset, concerning the set of targets of all graphs in \( G_i \). In this case, according to our algorithm, the cells in the row \( M[n_j][*] \) would have been annotated as $OR_C$, and all nodes \( n \in N_t \) would have been added in edges in \( G_U \) of the form \( (OR_C, n) \), with the edge \( (n_j, OR_C) \) also added to \( G_U \), thus entailing the inclusion of the path \( p(n_j, n) \) in \( G_U \). Consequently, in all the cases where the set of target nodes $N_t$ contains multiple elements, for every target node $n \in N_t$, there exists a gateway-mediated path $p(n_j,n$ in $G_U$ that contains either an $XOR_C$ gateway, an $OR_C^E$ gateway, or an $OR_C$ gateway, thus satisfying our condition for completeness.  

\end{proof}

%% file: main.bbl
\begin{thebibliography}{10}
\providecommand{\url}[1]{\texttt{#1}}
\providecommand{\urlprefix}{URL }
\providecommand{\doi}[1]{https://doi.org/#1}

\bibitem{vanderAalst2016ProcessMining}
van~der Aalst, W.: {Process Mining}. Springer, Berlin, Heidelberg (2016)

\bibitem{vanderAalst2011CausalDiscovery}
van~der Aalst, W., Adriansyah, A., van Dongen, B.: {Causal Nets: A Modeling Language Tailored towards Process Discovery}. In: CONCUR 2011 - Concurrency Theory. pp. 28--42. Springer (2011)

\bibitem{Alaee2024Data-DrivenAndDiscovery}
Alaee, A.J., Weidlich, M., Senderovich, A.: {Data-Driven Decision Support for Business Processes: Causal Reasoning and Discovery} (2024)

\bibitem{Bozorgi2020}
Bozorgi, Z.D., Teinemaa, I., Dumas, M., La~Rosa, M., Polyvyanyy, A.: {Process Mining Meets Causal Machine Learning: Discovering Causal Rules from Event Logs}. In: 2020 2nd International Conference on Process Mining (ICPM). pp. 129--136. IEEE (10 2020)

\bibitem{Conforti2017FilteringLogs}
Conforti, R., La~Rosa, M., Ter~Hofstede, A.H.: {Filtering Out Infrequent Behavior from Business Process Event Logs}. IEEE Transactions on Knowledge and Data Engineering  \textbf{29}(2) (2017)

\bibitem{Cunningham2021CausalMixtape}
Cunningham, S.: {Causal Inference: The Mixtape}. Yale University Press (2021)

\bibitem{Dempster1977Algorithm}
Dempster, A.P., Laird, N.M., Rubin, D.B.: { Maximum Likelihood from Incomplete Data Via the EM Algorithm }. Journal of the Royal Statistical Society Series B: Statistical Methodology  \textbf{39}(1) (1977)

\bibitem{Dumas2023}
Dumas, M., Fournier, F., Limonad, L., Marrella, A., {et al.}: {AI-augmented Business Process Management Systems: A Research Manifesto}. ACM Transactions on Management Information Systems  \textbf{14}(1) (2023)

\bibitem{Fournier2023v3}
Fournier, F., Limonad, L., Skarbovsky, I., David, Y.: {The WHY in Business Processes: Discovery of Causal Execution Dependencies}. K{\"{u}}nstliche Intelligenz  (1 2025)

\bibitem{Galanti2023AnAnalytics}
Galanti, R., de~Leoni, M., Monaro, M., Navarin, N., Marazzi, A., Stasi, B.D., Maldera, S.: {An explainable decision support system for predictive process analytics}. Eng. Appl. Artif. Intell.  \textbf{120},  105904 (2023)

\bibitem{Hernan2020CausalIf}
Hern{\'{a}}n, M.A., Robins, J.M.: {Causal Inference: What If}. CRC Press (12 2020)

\bibitem{Hompes2017DiscoveringVariation}
Hompes, B.F.A., et~al.: {Discovering Causal Factors Explaining Business Process Performance Variation}. In: Advanced Information Systems Engineering. pp. 177--192. Springer (2017)

\bibitem{Kourani2023MiningGraphs}
Kourani, H., Di~Francescomarino, C., Ghidini, C., van~der Aalst, W., van Zelst, S.: {Mining for Long-Term Dependencies in Causal Graphs}. In: Business Process Management Workshops. pp. 117--131. Springer, Cham (2023)

\bibitem{Leemans2022CausalModels}
Leemans, S.J.J., Tax, N.: {Causal Reasoning over Control-Flow Decisions in Process Models}. In: CAiSE 2022, Leuven, Belgium, June 6-10, 2022, Proceedings. LNCS, vol. 13295, pp. 183--200. Springer (2022)

\bibitem{Little2014StatisticalData}
Little, R.J., Rubin, D.B.: {Statistical analysis with missing data}. Wiley{\&}Sons (2014)

\bibitem{MichaelS.Dobson2004TheManagement}
{Michael S. Dobson}: {The Triple Constraints in Project Management }. Berrett-Koehler Publishers (7 2004)

\bibitem{Mohan2013MissingProblem}
Mohan, K., Pearl, J., Tian, J.: {Missing data as a causal inference problem}. In: Proceedings of the 31st Conference on Uncertainty in Artificial Intelligence (UAI) (10 2013)

\bibitem{Tanmayee2019}
Narendra, T., Agarwal, P., Gupta, M., Dechu, S.: {Counterfactual reasoning for process optimization using structural causal models}. In: Lecture Notes in Business Information Processing. vol.~360 (2019)

\bibitem{Pearl2011Causality:Edition}
Pearl, J.: {Causality: Models, reasoning, and inference, second edition}. Causality: Models, Reasoning, and Inference, Second Edition pp. 1--464 (1 2011)

\bibitem{Pearl2018Why}
Pearl, J., Mackenzie, D.: {The Book of Why: The New Science of Cause and Effect}. Basic Books, 1st edn. (5 2018)

\bibitem{Peters2017ElementsAlgorithms}
Peters, J., Janzing, D., Schlkopf, B.: {Elements of Causal Inference: Foundations and Learning Algorithms}. The MIT Press (2017)

\bibitem{Quine1955AFunctions}
Quine, W.V.: {A Way to Simplify Truth Functions}. The American Mathematical Monthly  \textbf{62}(9),  627--631 (11 1955)

\bibitem{Rubin1987MultipleSurveys}
Rubin, D.B.: {Multiple Imputation for Nonresponse in Surveys}. Wiley (6 1987)

\bibitem{Seaman2013ReviewData}
Seaman, S.R., White, I.R.: {Review of inverse probability weighting for dealing with missing data} (2013)

\bibitem{Shimizu2022StatisticalApproach}
Shimizu, S.: {Statistical Causal Discovery: LiNGAM Approach}. SpringerBriefs in Statistics, Springer Japan, Tokyo (2022)

\bibitem{Shimmura2013CircadianCrowing}
Shimmura, T., Yoshimura, T.: {Circadian clock determines the timing of rooster crowing} (2013)

\bibitem{ShoushMahmoud2022WhenConstraints}
{Shoush Mahmoud}, Dumas, M.: {When to Intervene? Prescriptive Process Monitoring Under Uncertainty and Resource Constraints}. In: Business Process Management Forum. pp. 207--223. Springer International Publishing, Cham (2022)

\bibitem{Spirtes2001CausationSearch}
Spirtes, P., Glymour, C., Scheines, R.: {Causation, Prediction, and Search}. The MIT Press (2001)

\bibitem{Yao2021AInference}
Yao, L., Chu, Z., Li, S., Li, Y., Gao, J., Zhang, A.: {A Survey on Causal Inference}. ACM Transactions on Knowledge Discovery from Data  \textbf{15}(5),  1--46 (10 2021)

\end{thebibliography}
